\documentclass[10pt,twocolumn,letterpaper]{article}

\usepackage[pagenumbers]{cvpr} %

\usepackage{soul}
\usepackage{amsthm}
\usepackage{url}
\usepackage{algorithm}
\usepackage{algpseudocode}%

\usepackage{graphicx}
\usepackage{float}
\usepackage{soul}

\usepackage{booktabs}

\usepackage[dvipsnames]{xcolor}

\usepackage[accsupp]{axessibility} %

\usepackage{pifont}%
\newcommand{\xmark}{\ding{55}}%

\newif\ifdraft
\draftfalse
\ifdraft
\newcommand{\gtc}[1]{\textcolor{cyan}{GT: #1}}
\newcommand{\gt}[1]{\textcolor{cyan}{#1}}
\newcommand{\gtdel}[1]{\textcolor{cyan}{\st{#1}}}

\newcommand{\abc}[1]{\textcolor{purple}{AB: #1}}
\newcommand{\rka}[1]{\textcolor{olive}{#1}}

\newcommand{\rkc}[1]{\textcolor{olive}{RK: #1}}
\newcommand{\rkdel}[1]{\textcolor{olive}{\st{#1}}}
\newcommand{\dcc}[1]{\textcolor{red}{DC: #1}}

\else
\newcommand{\gtc}[1]{}
\newcommand{\gt}[1]{#1}
\newcommand{\gtdel}[1]{}
\newcommand{\abc}[1]{}
\newcommand{\ysc}[1]{}

\newcommand{\ysdel}[1]{}
\newcommand{\rkc}[1]{}
\newcommand{\rka}[1]{#1}
\newcommand{\rkdel}[1]{}
\newcommand{\dcc}[1]{}

\fi

\usepackage{lipsum}

\newcommand{\ouralg}{MAS}
\newcommand{\ouralglong}{Multi-view Ancestral Sampling}
\newcommand{\ourmethod}{\ouralg{}}

\usepackage{amsmath,amsfonts,bm}

\def\eqref#1{equation~\ref{#1}}

\def\1{\bm{1}}

\DeclareMathAlphabet{\mathsfit}{\encodingdefault}{\sfdefault}{m}{sl}
\SetMathAlphabet{\mathsfit}{bold}{\encodingdefault}{\sfdefault}{bx}{n}

\definecolor{cvprblue}{rgb}{0.21,0.49,0.74}
\usepackage[pagebackref,breaklinks,colorlinks,citecolor=cvprblue]{hyperref}

\def\paperID{9444} %
\def\confName{CVPR}
\def\confYear{2024}

\title{\ouralg{}: Multi-view Ancestral Sampling \\ for 3D Motion Generation Using 2D Diffusion }

\author{%
\large
Roy Kapon, Guy Tevet, Daniel Cohen-Or and Amit H. Bermano\\
\normalsize{{Tel Aviv University}}\\
{\tt\small roykapon@mail.tau.ac.il}
}

\begin{document}
\twocolumn[{%
\renewcommand\twocolumn[1][]{#1}%
\maketitle
\vspace{-25pt}
\begin{center}
\centering
\captionsetup{type=figure}
\includegraphics[width=1\textwidth]{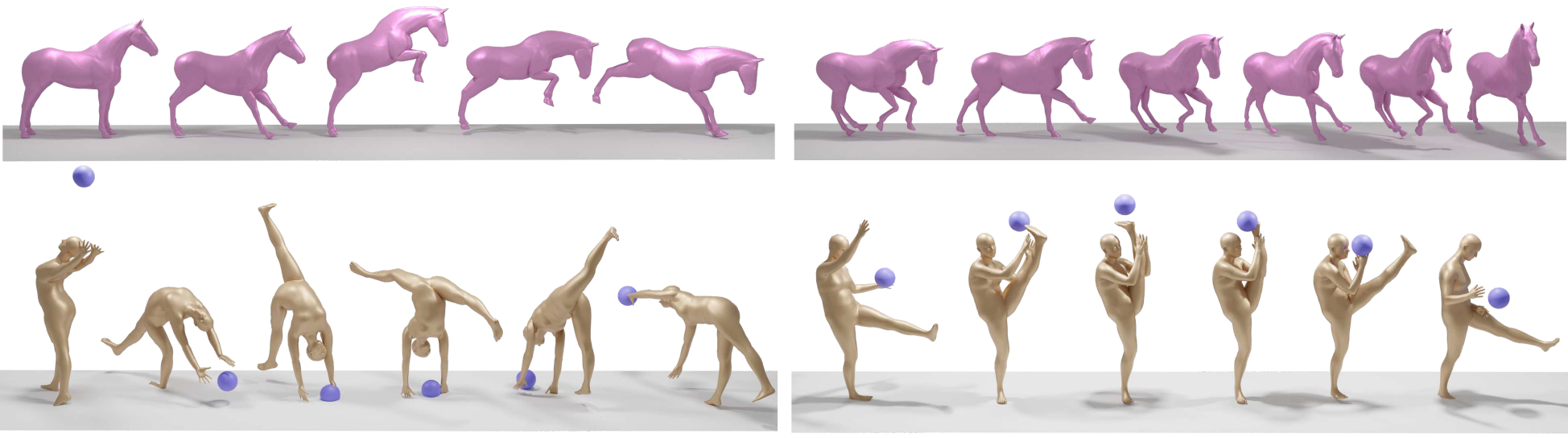}
\captionof{figure}{
3D motions generated by \ouralglong~(\ouralg{}) --- each one using a different initial noise.
Our method generates novel 3D motions using a 2D diffusion model. As such, it enables learning intricate 3D motion synthesis solely from monocular video data.
}
\label{fig:teaser}

\end{center}%
}]

\begin{abstract}
We introduce Multi-view Ancestral Sampling (MAS), a method for 3D motion generation, using 2D diffusion models that were trained on motions obtained from in-the-wild videos.
As such, 
\ouralg{}
opens opportunities to exciting and diverse fields of motion 
previously under-explored as 3D data is scarce and hard to collect. \ouralg{} works by simultaneously 
denoising multiple 2D motion sequences representing different views of the same 3D motion.
It 
ensures consistency across all views at each 
diffusion step by combining the individual generations into a unified 3D sequence, and projecting it back to the original views.
We demonstrate MAS on 2D pose data 
acquired from videos depicting professional basketball maneuvers, rhythmic gymnastic performances featuring a ball apparatus, 
and horse 
races.
In each of these domains, 3D motion capture is arduous, and yet, \ouralg{} generates diverse and realistic 3D sequences. 
Unlike the Score Distillation approach, which 
optimizes each sample by repeatedly applying small fixes, our method uses a sampling process that was constructed for the diffusion framework. 
As we
demonstrate, \ouralg{}
avoids
common issues such as out-of-domain sampling
and mode-collapse.
\href{https://guytevet.github.io/mas-page/.}{https://guytevet.github.io/mas-page/}

\end{abstract}

\section{Introduction}
\label{sec:intro}

3D motion generation is an increasingly popular field that has prominent applications in computer-animated films, video games, virtual reality, and more. One of the main bottlenecks of current approaches is reliance on 3D data, 
which is typically acquired by actors in motion capture studios or created by professional animation artists. 
Both forms of data acquisition are costly, not scalable, do not capture in-the-wild behavior, and leave entire motion domains under-explored.

Fortunately, the ubiquity of video cameras leads to countless high-quality recordings of a wide variety of motions. Naively, a possible way to leverage 
these videos for motion generation tasks is extracting 3D pose estimations and using them as training data. 
However, pose estimation methods are mostly trained using 3D data~\citep{vibe,shetty2023pliks}, thus inheriting the mentioned data limitations. Some methods only require 2D data~\citep{wandt2021elepose,deng2022svmac}, but suffer from noticeable artifacts and temporal inconsistencies.

Recently, ~\citet{azadi2023makeananimation} and \citet{zhang2023vid2player3d}  incorporated 3D motions estimated from images or videos into motion synthesis applications. 
The former used them to enrich an existing motion capture dataset and the latter as reference motions while learning a physically-based Reinforcement Learning policy. In both cases, the quality issues were bridged using strong priors (either high-quality 3D data or physical simulation), 
hence remaining limited to specific settings.
Contrary to the pose estimation approaches, we focus on unconditional 3D motion generation from pure noise.

In this paper, we present 
\ouralglong{} (\ouralg{}), 
a diffusion-based 3D motion generation method, requiring only 2D motion data that can be acquired exclusively from videos.
First, we learn a 2D motion diffusion model from a set of videos, then, 
we employ the \ouralg{} algorithm to effectively sample 3D motions from this learned model.
Our method is based on \emph{Ancestral Sampling} --- the standard denoising loop used for sampling from a diffusion model.
\ouralg{} extends this concept and generates a 3D motion by simultaneously denoising multiple 2D views describing it. At each diffusion denoising step, all views are triangulated into a single 3D motion and then projected back to each view. This ensures multi-view consistency throughout the denoising process, while adhering to the prior's predictions. We further encourage multi-view consistency by projecting a 3D noise to each view whenever sampling from a Gaussian distribution in the 2D ancestral sampling process.

We show that \ourmethod{} generates diverse and realistic motions from the underlying 3D motion distribution using a 2D diffusion model that was exclusively trained on motions obtained from in-the-wild videos.
Furthermore, we show that relying on ancestral sampling allows \ouralg{} to generate a 3D motion in a few seconds only, using a single standard GPU. 
\ourmethod{} excels in scenarios where acquiring 3D motion capture data is 
impractical while video footage is abundant (See Figure~\ref{fig:teaser}). In such settings, we apply off-the-shelf 2D pose estimators to extract 2D motion sequences from video frames, and use them to train our diffusion prior.
We demonstrate \ourmethod{} in three domains:
(1) professional basketball player motions extracted from common NBA match recordings, (2) horse motions extracted from equestrian contests, and (3) human-ball interactions extracted from rhythmic ball gymnastics performances (ball location is an additional parameter predicted by the model).
These datasets demonstrate motion domains that were previously under-explored due to 3D data scarcity.

\section{Related Work}
\label{sec:related}

\textbf{3D Motion Synthesis.}
Multiple works explore 3D motion generation using moderate-scale 3D motion datasets such as HumanML3D~\citep{Guo_2022_CVPR},  KIT-ML~\citep{plappert2016kit} Human3.6M ~\citep{6682899} and HumanAct12~\citep{guo2020action2motion}. 
With this data, synthesis tasks were traditionally learned using Auto-Encoders or VAEs~\citep{kingma2013auto}, \citep{holden2016deep,ahuja2019language2pose,petrovich22temos,Guo_2022_CVPR,tevet2022motionclip}.
Recently, Denoising Diffusion Models~\citep{sohl2015deep, song2020improved} were introduced to this domain by MDM~\citep{tevet2023human}, MotionDiffuse~\citep{zhang2022motiondiffuse}, MoFusion~\citep{dabral2023mofusion}, and FLAME~\citep{kim2022flame}.
Diffusion models were proven to have a better capacity to model the motion distribution of the data and provided opportunities for new generative tasks. Yet the main limitation of all the mentioned methods is their reliance on high-quality 3D motion capture datasets, which are hard to obtain and limited in domain and scale. 
In this context, SinMDM~\citep{raab2023single} enabled non-humanoid motion learning from a single animation; PriorMDM~\citep{shafir2023human} and GMD~\citep{karunratanakul2023gmd} presented fine-tuning and inference time applications for motion tasks with few to none training samples, relying on a pre-trained MDM.

\textbf{Monocular Pose Estimation.}
Monocular 3D pose estimation is a well-explored field~\citep{vibe,shetty2023pliks,yu2023glagcn,shan2023diffusionbased}. Its main challenge is the many ambiguities (e.g. self-occlusions and blurry motion) inherent to the problem.
A parallel line of work is pose lifting from 2D to 3D. MotionBERT~\citep{zhu2023motionbert} demonstrates a supervised approach to the task. %
Some works offer to only use 2D data and learn in an unsupervised manner; \citet{drover20183d} suggest training a 2D discriminator to distinguish between random projections of outputs of a 3D lifting network and the 2D data while optimizing the lifting network to deceive the discriminator; 
ElePose~\citep{wandt2021elepose} train a normalizing-flows model on 2D poses and then use it to guide a 3D lifting network to generate 3D poses that upon projection have high probability w.r.t the normalizing-flows model. They add self-consistency and geometric losses and also predict the elevation angle of the lifted pose which is crucial for their success.

\textbf{Animal 3D Shape Reconstruction.} 
The recent MagicPony~\citep{wu2023magicpony} estimates the pose of an animal given a single image by learning a per-category 3D shape template and per-instance skeleton articulations, trained to reconstruct a set of 2D images upon rendering.
\citet{yao2023artic3d}  suggest a method for improving the input images with occlusions/truncation via 2D diffusion. 
Then, they use a text-to-image diffusion model to guide 3D optimization process to obtain shapes and textures that are faithful to the input images.

\textbf{Text to 3D Scene Generation.}
DreamFusion \citep{poole2022dreamfusion} and SJC \citep{wang2022score}, introduced guidance of 3D content creation using diffusion models trained on 2D data.  \citet{poole2022dreamfusion} suggest SDS, a method for sampling from the diffusion model by minimizing a loss term that represents the distance between the model's distribution and the noised sample distribution. They suggest to harness SDS for 3D generation by repeatedly rendering a 3D representation (mostly NeRF~\citep{mildenhall2020nerf}  based) through a differentiable renderer, noising the resulting images using the forward diffusion, get a correction direction using the diffusion model, and then back-propagate gradients to update the 3D representation according to the predicted corrections.
Although promising, their results are of relatively low quality and diversity and suffer from slow inference speed, overly saturated colors, lack of 3D consistency, and heavy reliance on text conditioning.
Follow-up works such as ProlificDreamer~\citep{wang2023prolificdreamer}, HIFA~\citep{zhu2023hifa}, DreamTime~\citep{huang2023dreamtime}, DDS~\citep{hertz2023delta} and NFSD~\citep{katzir2023noise} expose those weaknesses and suggest various methods to mitigate them.
In a similar context, Instruct-NeRF2NeRF~\citep{instructnerf2023} edit a NeRF by gradually editing its source multi-view image dataset during training, using an image diffusion model. MVDream~\citep{shi2023mvdream} train a diffusion model to generate multiple views of the same object using a 3D object dataset. They apply SDS optimization loop using the diffusion model to correct multiple views of the optimizing object at each iteration. This method and similar ones \citep{liu2023zero1to3,yu2023pointsto3d,seo2023let,huang2023humannorm} 
heavily rely on additional data such as 3D structure, depth or normals, which is not available in our setting.

Contrary to the SDS approach which is an optimization process, our \ourmethod{} 
samples 3D motions from 2D diffusion models at inference. Hence it suggests a faster approach 
and avoids many of the SDS weaknesses by design (See Section~\ref{sec:analysis}).

\section{Preliminary}
\label{sec:preliminary}

\textbf{Diffusion Models and Ancestral Sampling.} Diffusion models are generative models that learn to gradually transform a predefined noise distribution into the data distribution. For the sake of simplicity, 
we consider the source distribution to be Gaussian. 
The forward diffusion process is defined 
by taking a data sample and gradually adding noise to it 
until we get a Gaussian distribution. 
The diffusion denoising model is then parameterized 
according to the reverse of this process, i.e. the model will sample a random Gaussian sample and gradually denoise it until getting a valid sample.

Formally, the forward process is defined by sampling a data sample $x_0\sim q\left(x_0\right)$ and for $t$ in $1,...,T$,  sampling $x_t\sim q\left(x_t|x_{t-1}\right)=\mathcal{N}(x_t;\sqrt{1-\beta_t}x_{t-1},\beta_t I)$, until getting to $x_T$, which has a gaussian distribution $x_T\sim q\left(x_T\right)=\mathcal{N}\left(x_T;0,I\right)$.  

The reverse process, also called \textbf{ancestral sampling}, is defined by sampling a random gaussian noise $x_T\sim p_\phi\left(x_T\right)=\mathcal{N}\left(x_T;0,I\right)$ and then for $t$ in $T,t-1,...,1$,  sampling $\hat{x}_{t-1}\sim p_\phi \left(\hat{x}_{t-1}|x_t\right)$ , until getting to $\hat{x}_0$, which 
should ideally approximate the data distribution.
The model posterior $p_\phi\left(x_{t-1}| x_t\right)$ is parameterized by a network $\mu_\phi\left(x_t, t\right)$:
\begin{align*}
p_\phi\left(x_{t-1}| x_t\right)=q\left(x_{t-1}|x_t,x_0=\mu_\phi\left(x_t;t\right)\right)=\\
\mathcal{N}\left(x_{t-1}; \mu_\phi\left(x_t, t\right), \sigma_t^2I\right)
\end{align*}
\vspace{1pt}
i.e. the new network predicts a mean denoising direction from $x_t$ which is then used
for sampling  $x_{t-1}$ from the posterior distribution derived from the forward process. 
$\mu_\phi$ is further parameterized
by a network $\epsilon_\phi$ that aims to predict the noise embedded in $x_t$: 
\begin{align*}
\mu_\phi(x_t, t) = \frac{1}{\sqrt{\alpha_t}}\left( x_t - \frac{\beta_t}{\sqrt{1-\bar\alpha_t}} \epsilon_\phi(x_t, t) \right) \label{eq:mu_func_approx_langevin}
\end{align*}

Now, when optimizing the usual variational bound on negative log-likelihood, it simplifies to,
\begin{equation*}
\mathcal{L}\left(\phi\right) = \mathbb{E}_{
\epsilon\sim \mathcal{N}\left(\mathbf{0}, \mathbf{I}\right) 
}
\left[w(t)\|\epsilon_\phi\left(\alpha_t x_0 + \sigma_t \epsilon; t\right) - \epsilon\|^2_2\right]
\label{eq:train}
\end{equation*}

which is used as the
training loss. We approximate this loss by sampling $t,\epsilon,x_0$ from their corresponding distributions and calculating the loss term. 
Note that when adding text-conditioning to the model, it is denoted by $p_\phi\left(x|y\right)$ where $y$ is the text prompt.

\textbf{Data Representation.}
A motion sequence is defined on top of a character skeleton with $J$ joints. 
A single character pose is achieved by placing each joint in space. Varying the character pose over time constructs a motion sequence.
Hence, we denote a 3D motion sequence, $X\in \mathbb{R}^{L\times J\times 3}$, with $L$ frames by the $xyz$ location of each joint at each frame.
Note that this representation is not explicitly force fixed bone length. Instead, our algorithm will do so implicitly.
Additionally, This formulation allows us to model additional moving objects in the scene (e.g. a ball or a box) using auxiliary joints to describe their location.

Considering the pinhole camera model\footnote{\url{https://en.wikipedia.org/wiki/3D_projection\#Perspective_projection}}, 
we define a camera-view $v=(R_v, \tau_v, f_v)$ by its rotation matrix $R_v \in \mathbb{R}^{3\times 3}$, translation vector $\tau_v \in \mathbb{R}^{3}$ and the focal length $f_v$ given in meters. Then, a 2D motion, 
$x^v = P(X, v) \in \mathbb{R}^{L\times J\times 2}$, 
from camera-view $v$, is defined as the perspective projection $P$ of $X$ to $v$ such that each joint at each frame is represented with its $uv$ coordinates of the camera space.

In order to drive 3D rigged characters (as presented in the figures of this paper) we retrieve 3D joint angles from the predicted 3D joint positions of $X$ using SMPLify~\citep{Bogo:ECCV:2016} optimization for human characters, and Inverse-Kinematics optimization for the non-humanoid characters (i.e. horses).

\section{Method}
\label{sec:method}

Our goal is to generate 3D motion sequences using a diffusion model trained on monocular 2D motions.
This would enable 3D motion generation in the absence of high-quality 3D data, by leveraging the ubiquity of monocular videos describing those scenes.
To this end, we introduce \ouralglong{} (\ouralg{}), a method that simultaneously generates multiple views of a 3D motion via ancestral sampling. \ouralg{} maintains consistency between the 2D motions in all views at each denoising step to construct a coherent 3D motion. A single \ouralg{} step is illustrated in Figure~\ref{fig:overview}.

In our experiments we first extract 2D pose estimations from in-the-wild videos and use them to train a 2D diffusion model $\hat{x}_0 = G_{2D}(x_t)$, that predicts the clean 2D motion, $\hat{x}_0$ at each denoising step (See Figure~\ref{fig:pipeline}).

\ouralg{} then uses the diffusion model to simultaneously apply an ancestral sampling loop on multiple 2D motions, which represent views of the same 3D motion from $V$ different camera angles. 
At each denoising step $t$, we get a set of noisy views $x^{1:V}_t$ as input and predict clean samples $\hat{x}^{1:V}_0 = G_{2D}(x^{1:V}_t)$. Then, the \emph{Consistency Block} is applied in two steps: (1) Triangulation: find a 3D motion $X$ that follows all views as closely as possible.
(2) Reprojection: project the resulting 3D motion to each view, getting $\tilde{x}^{1:V}_0$, which we can think of as a multiview-consistent version of the predicted motions. Finally, we can sample 
the next step
$x^{1:V}_{t-1}$ from the backward posterior $x^{1:V}_{t-1} \sim q\left(x_{t-1}|x_t,\tilde{x}^{1:V}_0\right)$ just like the original ancestral sampling algorithm. 
Repeating this denoising process up to $t=0$ yields multiple views of the same 3D motion. Finally, we triangulate the resulting 2D motions to create a 3D motion, which is returned as the final output. This sampling process is detailed in 
Algorithm~\ref{alg:3D sampling}. %
The remainder of this section describes the monocular data collection and diffusion pre-training (\ref{sec:prep}), followed by a full description of \ouralg{} building blocks (\ref{sec:alg}).

\subsection{Preparations}
\label{sec:prep}

\textbf{Data Collection.}
We collect videos from various sources ---
NBA videos, horse jumping contests, and rhythmic gymnastics contests. We then apply multi-person and object tracking using off-the-shelf models to extract bounding boxes. Subsequently, we use other off-the-shelf models for 2D pose estimation to get 2D motions. 
Implementation details are in Section ~\ref{sec:exp}. We build on the fact that 2D pose estimation is a well-explored topic, with large-scale datasets that can be easily scaled as manual annotations are much easier to obtain compared to 3D annotation which usually requires a motion capture studio.

 \textbf{2D Diffusion Model Training.}
We follow \citet{tevet2023human} and train the unconditioned version of the Motion Diffusion Model (MDM) with a transformer encoder backbone for each of the datasets separately. We boost the sampling of MDM by a factor of $10$ by learning $100$ diffusion steps instead of the original $1000$. %

\subsection{\ouralglong{}}
\label{sec:alg}
We would like to construct a way to sample a 3D motion using a model that generates 2D samples. First, we observe that a 3D motion is uniquely defined by 2D views of it from multiple angles. Second, we assume that our collected dataset includes a variety of motions, from multiple view-points, and deduce that our 2D diffusion model can generalize for generating multiple views of the same 3D motion, for a wide variety of 3D motions.
Thus, we aim to generate multiple 2D motions that represent multiple views of the same 3D motion, from a set of different view-points.

\begin{figure}[t!]
\begin{center}
\includegraphics[width=\columnwidth]{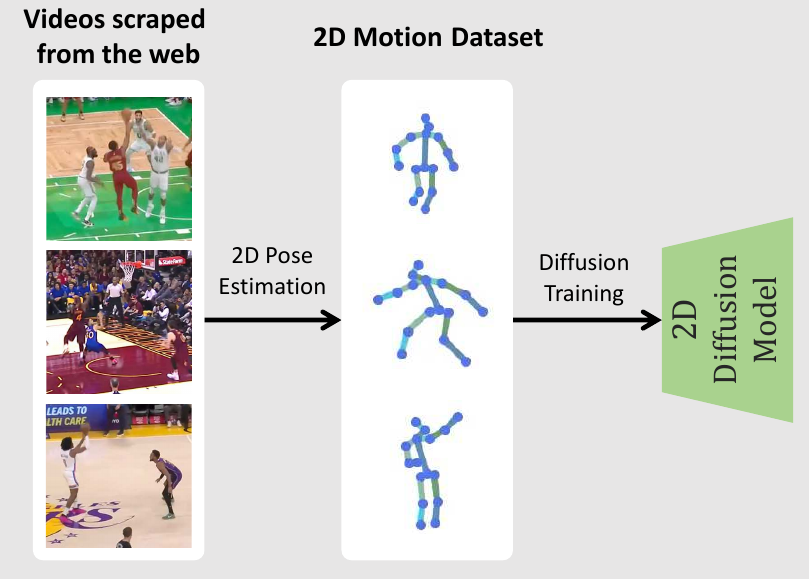}
\end{center}
\caption{
\textbf{Preparations.} The motion diffusion model used for \ouralg{} is trained on 2D motion estimations of videos scraped from the web. 
}
\label{fig:pipeline}
\end{figure}

\begin{figure*}
\begin{center}
\includegraphics[width=\textwidth]{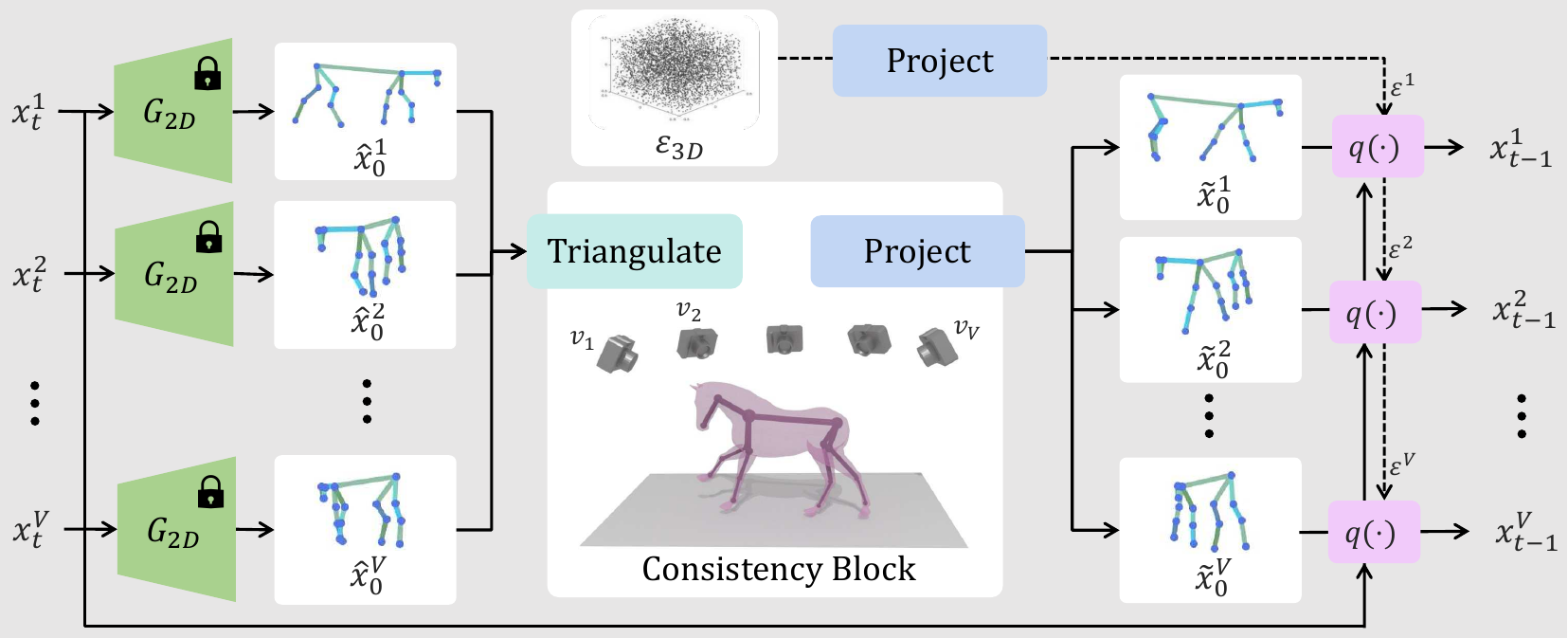}
\end{center}
\caption{
The figure illustrates an overview of \ouralg{}, showing a multi-view denoising step from the 2D sample collection $x_t^{1:V}$ to $x_{t-1}^{1:V}$, corresponding to camera views ${v}_{1:V}$. Denoising is performed by a fixed 2D motion diffusion model $G_{2D}$. At each such iteration, our \emph{Consistency Block} triangulates the motion predictions $\hat{x}_0^{1:V}$ into a single 3D sequence and projects it back onto each view ($\tilde{x}_0^{1:V}$). To encourage consistency in the model's predictions, we sample 3D noise, $\epsilon_{3D}$ and project it to the 2D noise ${\epsilon^{v}}$ for each view. Finally, we sample $x_{t-1}^{1:V}$ from $q\left(x^{1:V}_{t-1}|x^{1:V}_t,\tilde{x_0}^{1:V}\right)$.
}
\label{fig:overview}
\end{figure*}
\textbf{Ancestral Sampling for 3D generation.}
As described in 
Section~\ref{sec:preliminary}, diffusion models are designed to be sampled using gradual denoising, following the ancestral sampling scheme. Hence, we design \ouralg{} to generate multiple 2D motions via ancestral sampling, while guiding all views to be multiview-consistent. Formally, we take a set of $V$ views, distributed 
evenly around the motion subject, with elevation angle distribution heuristically picked for each dataset. Then, for a each view $v$ we initialize $x_T^v$ with noise, and for $t=T,...,1$ transform $x_t^v$ to $x_{t-1}^v$ until getting a valid 2D motion $x_0^v$ for each view. We choose to generate all views concurrently, keeping all views in the same diffusion timestep throughout the process.

 In every denoising step we receive $x_t^{1:V}=\left(x_t^1,...,x_t^V\right)$. We derive the clean motion predictions by applying the diffusion model in each view $\hat{x}_0^v:=\frac{x_{t}^v-\sqrt{1-\bar{\alpha}_{t}}\epsilon_{\phi}\left(x_t^v\right)}{\sqrt{\bar{\alpha}_{t}}}$, getting $\hat{x}_0^{1:V}=\left(\hat{x}_0^1,...,\hat{x}_0^V\right)$. 
 We apply our multi-view Consistency Block to find multi-view consistent motions $\tilde{x}_0^{1:V}$ that approximate the predicted motions $\hat{x}_0^{1:V}$. 
 We then use the resulting motions  $\tilde{x}_0^{1:V}$ as the denoising direction by sampling $x_{t-1}^v$ from $q\left( x_{t-1}^v | x_t^v, x_0=\tilde{x}_0^v \right)$, and outputting $x_{t-1}^{1:V}=\left(x_{t-1}^1,...,x_{t-1}^V\right)$.

\ouralg{} can be extended to support dynamic camera-view along sampling instead of fixed ones as detailed in 
Appendix~\ref{app:dynamic}. 
Since this is not empirically helpful for our application, we leave it out of our scope.

 \textbf{Multi-view Consistency Block}
As mentioned, the purpose of this block is to transform multiview motions $\hat{x}_0^{1:V}$ into multiview-consistent motions $\tilde{x}_0^{1:V}$ that are as similar as possible. We achieve this by finding a 3D motion $X$ that when projected to all views, it resembles the multiview motions $\hat{x}_0^{1:V}$ via \emph{Triangulation}. 
We then return projections of $X$ to each view  
$\tilde{x}_0^{1:V} = \left(P\left( X,1 \right),...,P\left( X,V \right)\right)$,
as the multiview-consistent motions. Since the denoising process is gradual, the model's predictions are approximately multiview-consistent so the consistency block only makes small corrections.

\textbf{Triangulation.}
We calculate $X$ via optimization to minimize the difference between projections of $X$ to all views and the multiview motion predictions $\hat{x}_0^{1:V}$:

\begin{align*}
X =
\underset{X'}{\arg\min}\lVert P\left( X', 1\!\!:\!\!V\right)-\hat{x}_0^{1:V}\rVert_2^2 
=\\
\underset{X'}{\arg\min} \sum_{v=1}^{V}\lVert P\left(X',v\right)  -\hat{x}_0^{v}\rVert_{2}^{2}
\end{align*}

For faster convergence, we initialize $X$  with the optimized results from the previous sampling step. This way the process can also be thought of progressively refining $X$ but we wish to emphasize that the focus remains the ancestral sampling in the 2D views.

\textbf{3D Noise.}
When triangulating the 2D motions $\hat{x}_0^{1:V}$, we would like them to be as close to being multiview-consistent as possible. A critical observation is that for our model to generate multiview-consistent motions we would like to pass it multiview-consistent noised motions. To this end, we design a new noise sampling mechanism that will (1) keep Gaussian distribution for each view, and (2) maintain multiview-consistency.

We start by sampling 3D noise $\varepsilon_{3d}\sim \mathcal{N}\left( 0, I\right)$ ($\varepsilon_{3d}\in \mathbb{R}^{L\times J\times 3}$). Projecting this noise to each view using perspective projection will result in a distribution that is not Gaussian. Hence, we instead use orthographic projection, which preserves Gaussian distribution for each view 
(see Appendix~\ref{appendix:theorems},\ref{theorem: multiview noise distribution}), 
and can differ from perspective projection by at most $O\left(1/\left(d-1\right)\right)$, where $d$ is the distance between the camera and the subject's center and assuming the subject is normalized to be bounded in a sphere with radius $1$
 (see Appendix~\ref{appendix:theorems},\ref{theorem: orthographic vs perspective}). 
We then use the resulting distribution
for sampling the initial noise $x_T$ and when sampling 
$x_{t-1}\sim q\left( x_{t-1} | x_t, x_0=P\left( X\right) \right)$ 
which significantly improves the quality and diversity of our results (see table \ref{table:ablation}).

\begin{figure}
\begin{center}
\includegraphics[width=1\columnwidth]
{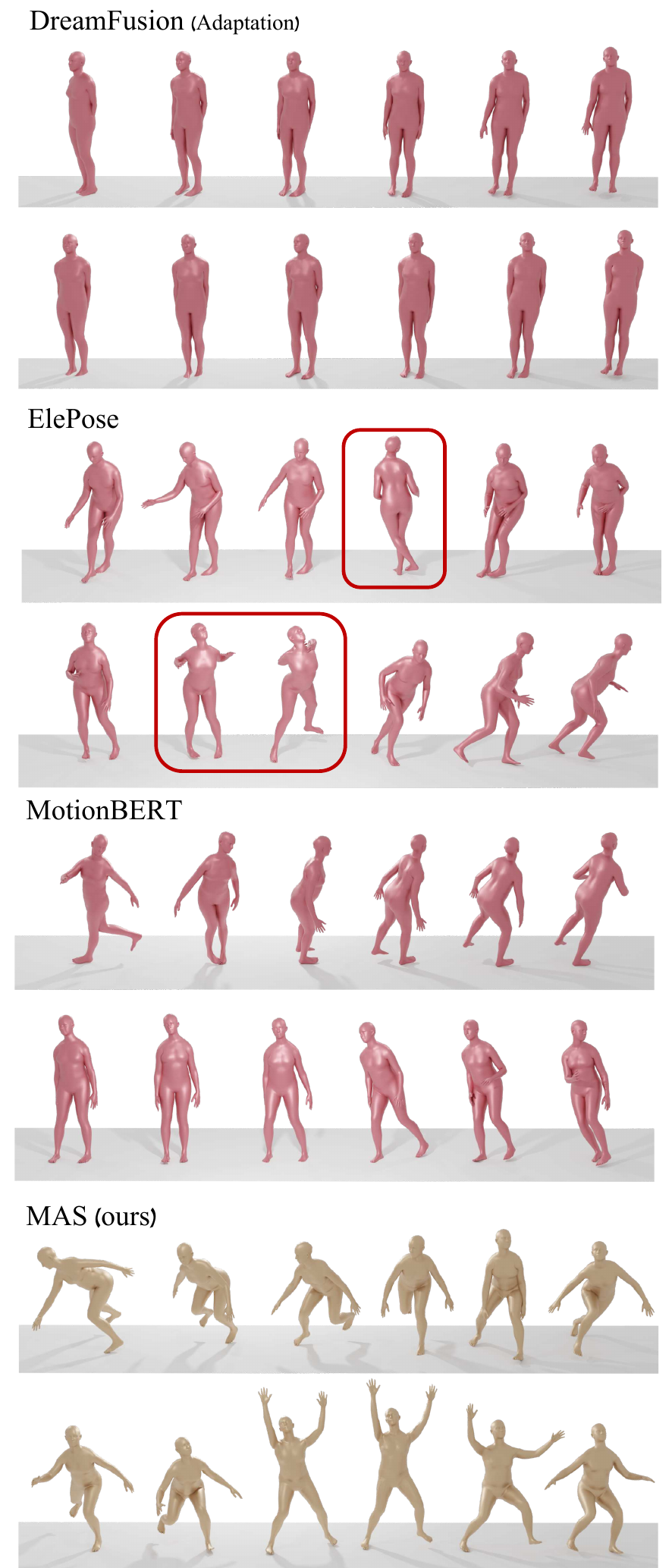}
\end{center}
\caption{
Generated motions by \ouralg{} compared to ElePose~\citep{wandt2021elepose}, 
MotionBert~\citep{zhu2023motionbert},
and an adaptation of
DreamFusion~\citep{poole2022dreamfusion}  to unconditioned motion generation.
We observe that MotionBert and DreamFusion produce dull motions with limited movement and ElePose predictions are jittery and often include invalid poses (Red rectangles).
}
\label{fig:comparison}
\end{figure}

\begin{figure}
\includegraphics[width=\columnwidth]{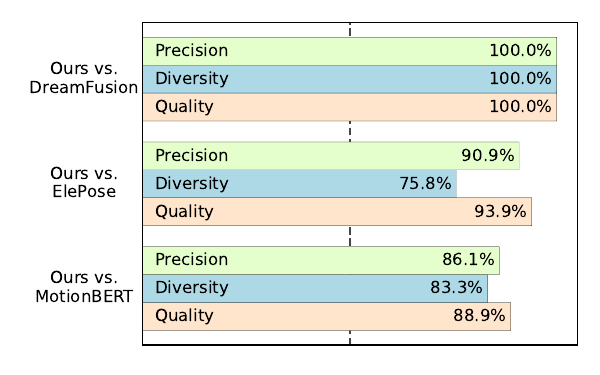}
\caption{
\textbf{NBA Dataset User study.} 
We asked $22$ unique users to compare $15$ randomly generated motions by each of the models to MAS generations in $3$ aspects - \emph{precision} (i.e. what samples best depict Basketball moves), Overall \emph{Quality} and \emph{Diversity}.
The dashed line marks $50\%$. \ouralg{} outperforms the lifting methods and the DreamFusion adaptation.
} \label{fig:user_study}
\end{figure}

\begin{algorithm*}[t]
\caption{\ouralglong{} (\ouralg{})}
\begin{algorithmic}
\State \textbf{Sample camera views:} $v_{1:V}\sim \mathcal{V}$

\State\textbf{Initialize 3D noise:} $\varepsilon_{3D} \sim \mathcal{N}_{L \times J \times 3}\left(0,I\right)$
\State\textbf{Initialize views by projection:} $x^{1:V}_T = P\left(\varepsilon_{3D},v_{1:V}\right)$
\For{$t=T,T-1,...,0$}
\State $\hat{x}^{1:V}_0 = G_{2D} \left( x^{1:V}_t\right)$
\State\textbf{Triangulate:} $X = \mathop{\mathrm{argmin}}\limits_{X'\in \mathbb{R}^{L\times J\times 3}} || P\left(X', v_{1:V}\right) -  \hat{x}^{1:V}_0||^2_2$
\Comment{$X, \varepsilon_{3D} \in \mathbb{R}^{L\times J\times 3}$}
\State \textbf{Back-project:} $\tilde{x}^{1:V}_0 = P(X, v_{1:V})$ 

\Comment{$x^{1:V}_{t}, \hat{x}^{1:V}_0, \tilde{x}^{1:V}_0, \varepsilon^{1:V} \in \mathbb{R}^{V\times L\times J\times 2}$}

\State \textbf{Sample noise:} $\varepsilon_{3D} \sim \mathcal{N}_{L \times J \times 3}\left(0,I\right)$ 

\State \textbf{Project noise:} $\varepsilon^{1:V} = P\left( \varepsilon_{3D}, v_{1:V}\right)$

\State \textbf{Denoising step:} $x^{1:V}_{t-1}=\frac{\beta_t \sqrt{\bar{\alpha}_{t-1}}}{1-\bar{\alpha}_t}x^{1:V}_t + \frac{\left( 1-\bar{\alpha}_{t-1} \right) \sqrt{\alpha_t}}{1-\bar{\alpha}_t}\tilde{x}^{1:V}_{0} + \frac{\beta_t \left( 1-\bar{\alpha}_{t-1} \right)}{1-\bar{\alpha_t}}\varepsilon^{1:V} $
 \EndFor
\State \textbf{Output triangulation:} 
$\mathop{\mathrm{argmin}}\limits_{X'} || P(X', v_{1:V}) - x^{1:V}_0 ||^2_2$
\end{algorithmic}
\label{alg:3D sampling}
\end{algorithm*}

\begin{table*}[t]
\centering
\vspace{10pt}
\resizebox{0.9\textwidth}{!}{
\begin{tabular}{ lcccccc  } %
\toprule
 Dataset Name & Subject & \#Samples & Length Range & Average Length & FPS & In-the-wild videos\\
\hline
Human3.6M~\citep{6682899} & Humans& 300& 42s-240s& ~104s&25 & \xmark \\
NBA videos & Humans & $ 60K$ &  $4$s-$16$s & $ 6$s & $30$ & \checkmark \\
 Horse jumping contests & Horses & $ 2K$ &  $3$s-$40$s & $ 7$s & $20$ & \checkmark \\
Rhythmic ball gymnastics & Humans + Ball & $ 500$ & $10$s-$120$s & $81$s & $20$ & \checkmark \\
 \bottomrule

\end{tabular}
}

\caption{\textbf{2D Datasets.} 
Details of the 2D motion datasets used for our experiments. The last three are newly collected in-the-wild datasets which we made available at \href{https://guytevet.github.io/mas-page/.}{https://guytevet.github.io/mas-page/}. }
\label{tab:datasets}
\vspace{-15pt}
\end{table*}

\section{Method Discussion}
\label{sec:analysis}

In this section, we discuss the properties of \ouralg{}, contextualizing it within the landscape of recent advancements in the text-to-3D domain.

\textbf{Ancestral sampling.} 
\ouralg{} is built upon the ancestral sampling process.
This means that the model is used in its intended way over in-domain samples.
This is in contrast to SDS-based methods~\citep{poole2022dreamfusion}  which 
employ a sampling scheme that uses the forward diffusion to noise images rendered from a 3D representation that is only partially optimized. This can lead to out-of-distribution samples, particularly 
when using smaller timesteps where the model expects motions that are close to being real.
This phenomenon is also addressed by \cite{wang2022score} and \cite{huang2023dreamtime}, who suggest heuristics to alleviate the out-of-distribution problem but do not fundamentally solve it.
Furthermore, most SDS-based methods sample $x_t$ independently in each iteration, which may lead to a high variance in the correction signal. Contrarily, using ancestral sampling has, by definition, a large correlation between $x_t$ and $x_{t-1}$, which leads to a more stable process and expressive results. Since \ouralg{} is sampling-based, it naturally models the diversity of the distribution, while optimization-based methods often experience mode-collapse or
divergence, as addressed by \cite{poole2022dreamfusion}.
It is worth noting that SDS is a clever design for cases where ancestral sampling cannot be used. 

\textbf{Multi-view stability.}
\ouralg{} simultaneously samples multiple views that share the same timestep at each denoising step. 
SDS-based methods typically use a single view in each optimization step, forcing them to make concessions such as small and partial corrections to prevent ruining the 3D object from other views. This also leads to a state where it is unknown which timestep to choose, since only partial denoising steps were applied (also shown by \cite{huang2023dreamtime}). \ouralg{} avoids such problems since the multiview denoising steps are applied simultaneously. It allows us to apply full optimization during the triangulation process. Hence, by the end of the $i$'th iteration, each view follows the model's distribution at timestep $T-i$. This alleviates the need for timestep scheduling and avoids out-of-distribution samples.

\textbf{3D noise consistency.}
\ouralg{}'s usage of a multiview-consistent noise distribution, 
critically boosts multiview-consistency in the model's predictions and greatly benefits the quality and diversity of the generated motions.
SDS-based methods sample uncorrelated noise in different views, which leads to inconsistent corrections, that can result in a lack of 3D consistency, slower convergence or even divergence.

\section{Experiments}
\label{sec:exp}

\subsection{Data Collection}

In order to demonstrate the merits of our method, we apply \ouralg{} on three different 2D motion datasets. Each dataset addresses a different motion aspect that is under-represented in existing 3D motion datasets (See Table~\ref{tab:datasets}).
(1) The NBA players' performance dataset demonstrates motion generation in domains of human motions that are 
poorly covered by existing datasets. 
(2) The horse show-jumping contests dataset shows generation in a domain that has almost no 3D data at all and has a completely different topology. 
Finally, (3) the rhythmic-ball gymnastics dataset shows that our method opens the possibility to model 
interactions with dynamic objects. All datasets include motions from diverse views, which is crucial for the success of our method. We detail the data collection process in 
Appendix~\ref{sec:data_collect}.

\gt{In addition, we evaluate \ouralg{} on the 3D motion dataset, Human3.6M~\citep{6682899}, by projecting the motions to random 2D cameras.}

All motions are represented as $x \in \mathbb{R}^{L\times J\times 2}$ as was detailed in Section~\ref{sec:preliminary}, where NBA is using the AlphaPose body model with $16$ joint, horses represented according to APT-36K with $17$ joints and the gymnastics dataset is represented with the COCO body model~\citep{lin2015microsoft} with $17$ joints plus additional joint for the ball. 
All 2D pose predictions are accompanied by confidence predictions per joint per frame which are used in the diffusion training process.

\subsection{Implementation Details}
Our 2D diffusion model is based on MDM~\citep{tevet2023human}, and composed of a transformer encoder with 6 attention layers of 4 heads and a latent dimension of 512. 
This backbone supports motions with variable length in both training and sampling, which makes \ouralg{} support it as well.
To mitigate some of the pose prediction errors, we mask low-confidence joint predictions from the training loss.
We used an ADAM optimizer with $10^{-4}$ lr for training and cosine noise scheduling. We learn $100$ diffusion steps instead of $1000$ which accelerate \ouralg{} $10$-fold without compromising the quality of the results.
We observe that \ouralg{} performs similarly for any $V\geq3$ and report $5$ camera views across all of our experiments.
The camera views $v_{1:V}$ are fixed through sampling,  surrounding the character and sharing the same elevation angle, with azimuth angles evenly spread around $[0,2\pi]$.
Generating a 3D sample with \ouralg{} takes less than $10$ seconds on a single NVIDIA GeForce RTX 2080 Ti. 
\gt{Performance details can be found in the 
Appendix~\ref{sec:perf}.}

\begin{table*}[t]
\vspace{10pt}
\resizebox{0.8\textwidth}{!}{
\begin{tabular}{ lcc|cc|cc|cc  }
\toprule
 & \multicolumn{2}{c}{FID$\downarrow$} & \multicolumn{2}{c}{Diversity$\rightarrow$} & \multicolumn{2}{c}{Precision$\uparrow$} & \multicolumn{2}{c}{Recall$\uparrow$} \\
 \hline
 View Angles  & All & Side  & All & Side  & All & Side  & All & Side  \\
 \hline
Ground Truth 
& \multicolumn{2}{c|}{$1.05^{\pm.02}$}
& \multicolumn{2}{c|}{$8.97^{\pm.05}$}
& \multicolumn{2}{c|}{$0.73^{\pm.01}$}
& \multicolumn{2}{c}{$0.73^{\pm.01}$}
\\
 \hline
 ElePose~\citep{wandt2021elepose}                            
 & $10.76^{\pm.45}$ & $18.28^{\pm.33}$ %
 & $9.72^{\pm.05}$ & $\mathbf{8.98^{\pm.06}}$ %
 & $0.28^{\pm.02}$ & $0.26^{\pm.02}$ %
 & $0.58^{\pm.03}$ & $0.17^{\pm.01}$ %
 \\
MotionBert~\citep{zhu2023motionbert}
& $30.22^{\pm.26}$ & $36.89^{\pm.40}$ %
& $9.57^{\pm.09}$ & $8.67^{\pm.08}$ %
& $0.04^{\pm4e-03}$ & $0.03^{\pm.01}$ %
& $0.34^{\pm.04}$ & $0.15^{\pm.04}$ %
\\
\ouralg{} (Ours) & \multicolumn{2}{c|}{$\mathbf{5.38^{\pm.06}}$}
& \multicolumn{2}{c|}{$9.47^{\pm.06}$}
& \multicolumn{2}{c|}{$\mathbf{0.50^{\pm.01}}$}
& \multicolumn{2}{c}{$\mathbf{0.60^{\pm.01}}$}
\\
 \hline
\end{tabular}
}
\centering
\caption{\textbf{Comparison with pose lifting on NBA dataset.} \ouralg{} outperforms state-of-the-art unsupervised lifting methods. Furthermore, lifting methods experience a drop in recall when evaluated from the side view ($\mathcal{U}\left(\frac{\pi}{4}, \frac{3\pi}{4}\right)$), while \ouralg{} does not suffer from this limitation as it is a generative approach, and not lifting-based.
`$\rightarrow$' means results are better when the value is closer to the real distribution ($8.97$ for Diversity); \textbf{bold} marks best results.
}
\label{table:comparison_nba}
\vspace{-10pt}
\end{table*}
\begin{table}[t]
\centering
\vspace{10pt}
\resizebox{1\columnwidth}{!}{
\begin{tabular}{lcccc}
\toprule
& FID$\downarrow$ & Diversity$\rightarrow$ & Precision$\uparrow$ & Recall$\uparrow$ \\
\hline
Ground Truth 
& $1.05^{\pm.02}$  
& $8.97^{\pm.05} $
& $0.73^{\pm.01}$ 
& $0.73^{\pm.01} $
\\
2D Diffusion Model & 
$5.23^{\pm.13}$ 
& $9.70^{\pm.08}$ 
& $0.44^{\pm.02}$ 
& $0.78^{\pm.01}$ 
\\
\hline
\ouralg{} (Ours) 
& $\mathbf{5.38^{\pm.06}}$ %
& $\mathbf{9.47^{\pm.06}}$ %
& $\mathbf{0.50^{\pm.01}}$ %
& $0.60^{\pm.01}$ %
\\
with $2$ views ($120^\circ$) 
& $6.87^{\pm.14}$ 
& $9.99^{\pm.06}$ 
& $0.35^{\pm.01}$ 
& $\mathbf{0.80^{\pm.01}}$ 
\\
\textcolor{lightgray}{- 3d noise}
& \textcolor{lightgray}{$17.40^{\pm.12}$} %
& \textcolor{lightgray}{$6.67^{\pm.07}$} %
& \textcolor{lightgray}{$0.93^{\pm.01}$} %
& \textcolor{lightgray}{$0.01^{\pm2.6e-03}$} %
\\
\hline
DreamFusion~\citep{poole2022dreamfusion}
& $66.38^{\pm1.24}$ %
& $8.25^{\pm.16}$ %
& $0.33^{\pm.08}$ %
& $0.17^{\pm.13}$ %
\\
\bottomrule
\end{tabular}
}
\caption{
\textbf{Ablations.} We compare \ouralg{} to an adaptation of DreamFusion~\citep{poole2022dreamfusion} to the unconditional motion generation domain. Our evaluation measures the quality of 2D projections of the 3D generated motions. Our ablations show that \ouralg{} performs best with as few as $5$ views (ours), and 3D noise is crucial for preventing mode collapse. \textcolor{lightgray}{gray} indicates mode-collapse (Recall$<10\%$), \textbf{bold} marks the best results otherwise. `$\rightarrow$' means results are better when the value is closer to the real (train data) distribution.
}
\label{table:ablation}
\vspace{-15pt}
\end{table}

\subsection{Evaluation}
 Here we explore the quality of the 3D motions generated by our method. Our experiments are conducted on the NBA dataset to allow comparison with existing methods, which mostly explore human motion. Usually, we would compare the generated motions to motions sampled from the dataset. In our case, we do not have 3D data so we must introduce a new way to evaluate the 3D generated motions. 
 For that sake, we rely on the assumption that a 3D motion is of high-quality if and only if all 2D views of it are of high-quality. 
 Consequently, we suggest taking random projections of the 3D motions and comparing them with our 2D data. 
 More specifically, we generate a set of 3D motions, with lengths sampled from the data distribution, then sample a single angle for every motion with yaw drawn from $\mathcal{U}\left[0,2\pi\right]$ and a constant pitch angle fitted for each dataset. We project the 3D motion to the sampled angle using perspective projection, from a constant distance (also fitted for each dataset) and get a set of 2D motions.

Finally, we follow common evaluation metrics~\citep{raab2022modi,tevet2023human} used for assessing unconditional generative models: 
    \emph{FID} measures Fréchet inception distance between the generated data distribution and the test data distribution;
    \emph{Diversity} measures the variance of generated motion
in latent space;
    \emph{Precision} measures the portion of the generated data that is covered by the test data;
    \emph{Recall} measures the portion of the test data distribution that is covered by the measured distribution.
These metrics are predominantly calculated in latent space. Hence, we train a VAE-based evaluator for each dataset.
We evaluate over $1K$ random samples and repeat the process $10$ times to calculate the average value and confidence intervals.
Table~\ref{table:ablation} shows that \ouralg{} results are comparable to the diffusion model in use, which marks a performance upper bound in 2D. We show that the addition of the multiview-consistent noise is crucial to the success of our method and prevents mode collapse. 
\gt{
A thorough ablation study for the number of views, camera distance, and number of diffusion steps can be found in 
Appendix~\ref{sec:add_results}.
}

We evaluate an adaptation of DreamFusion~\citep{poole2022dreamfusion} to the unconditioned motion generation domain and show that it performs poorly. 
\rka{This is carried out by initializing a random 3D motion and then performing 200 SDS iterations using the same diffusion model we used for \ouralg{}. Each iteration is comprised of: (1) Projecting the 3D motion to some random view (view distribution is the same as in \ouralg{}).(2) Noising the resulting 2D motion to some diffusion timestep $t \sim \mathcal{U}[1, T]$. (3)  Letting our diffusion model predict a cleaner version of the noised motion. (4) Updating the 3D motion to fit the predicted motion in the sampled view using a single optimization step. The implementation of this adaptation can be found in our published code.}

We also experimented with higher iteration numbers than 200 and techniques such as timestep scheduling and optimization tuning but saw no significant improvement.

We compare our method with off-the-shelf SOTA methods for supervised pose lifting - MotionBERT~\citep{zhu2023motionbert} - and unsupervised pose lifting - ElePose~\citep{wandt2021elepose}. Although these methods are not generative per-se, we consider lifted motions from 2D motions sampled from the training data as generated samples. As Elepose only requires 2D data, we train it on our NBA dataset and adjust the geometric priors to our data. MotionBert was trained on Human3.6M ~\citep{6682899} dataset and some in-the-wild videos, so it is applied in a zero-shot setting.
Table~\ref{table:comparison_nba} shows that \ouralg{} outperforms both lifting methods.

Since we sample a uniform angle around the lifted motions, we often project them to views that are similar to the lifted view. This results in a motion that resembles the lifted motion, which was sampled from the train data, thus boosting performance.
We show that when evaluating from the side view ($\text{angle}\sim \mathcal{U}\left(\frac{\pi}{4}, \frac{3\pi}{4}\right)$ relative to the lifting angle)), the lifting methods experience a clear degradation in performance. \ouralg{} is unaffected as it is a generative approach and has no "side" view.
\gt{
Repeating this experiment with the 3D dataset Human3.6M, randomly projected into 2D cameras shows that \ouralg{} is on par with the side-view performance of MotionBERT, and ElePose. More details in 
Appendix~\ref{sec:add_results}.
}

Figure~\ref{fig:comparison} demonstrates the quality of \ouralg{} compared to DreamFusion, MotionBERT, and ElePose.
\gt{
Figure~\ref{fig:user_study} presents a user study conducted with $22$ participants comparing $15$ randomly generated 3D motions by each of the models.
An example screenshot from the study can be found in 
Appendix~\ref{sec:add_results}.
}

\section{Conclusions}
\label{sec:discussion}

In this paper, we introduced \ouralg{}, a generative 
method designed for 3D motion synthesis using 2D data.
We showed that high-quality 3D motions can be sampled from a diffusion model trained on 2D data only.
The essence of our method lies in its utilization of a multiview diffusion ancestral sampling process, where each denoising step contributes to forging a coherent 3D motion sequence.

Our experiments show that \ouralg{} excels with in-the-wild videos, enabling it to produce motions that are otherwise exceedingly challenging to obtain through conventional means.

Our method could also be employed in additional domains such as multi-person interactions, hand and face motions, complex object manipulations and with recent developments in tracking of ``any" object~\citep{wang2023tracking}, we wish to push the boundaries of data even further.

\rka{Our method does experience some failure cases: The character occasionally folds into itself when changing direction, and the character sometimes changes its scale throughout the motion.} \ouralg{} \rka{also} inherits the limitations of the 2D data it is using and thus cannot naively predict global position, or apply textual control. We leave extending the data acquisition pipeline to support such features to future work. It is also worth noting that our method requires 2D data that captures a variety of views of similar motions. Finally, we hope the insights introduced in this paper can also be utilized in the text-to-3D field and other applications.

\section*{Acknowledgements}
We thank Elad Richardson, Inbar Gat, and Matan Cohen for thoroughly reviewing our early drafts.
We thank Sigal Raab, Oren Katzir, and Or Patashnik for the fruitful discussions.
This research was supported in part by the Israel Science Foundation (grants no. 2492/20 and 3441/21), Len Blavatnik and the Blavatnik family foundation, and The Tel Aviv University Innovation Laboratories (TILabs). This work was supported by the Yandex Initiative in Machine Learning.

{
    \small
    \bibliographystyle{ieeenat_fullname}
    \bibliography{main_bib}
}

\clearpage
\appendix
\section*{\textbf{Appendix}}

\section{Performance Details}
\label{sec:perf}
Table~\ref{table:costs} displays the time needed for a single sample generation and the GPU memory it consumes.

\begin{table}[H]
\centering
\resizebox{0.9\columnwidth}{!}{
\begin{tabular}{lcccc}
\toprule
& MAS & DreamFusion & ElePose & MotionBert \\
\hline
Time[$sec$]
& $9$  
& $17$
& $2.3\cdot 10^{-3}$ 
& $1$
\\
\hline
Memory[$MB$]&$794$&$794$&$686$&$784$\\
\bottomrule
\end{tabular}
}
\caption{
\small
Time and memory costs per single sample generation.
}
\label{table:costs}
\end{table}

\section{Dynamic View-point Sampling}
\label{app:dynamic}
Keeping the optimized views constant could theoretically lead to overfitting a motion to the optimized views, while novel views might have a lower quality. Note that this problem arises only at a lower number of views ($<5$). For this reason, we suggest a way to re-sample the viewing-points: After every step, we can save $X^{\left(i\right)}$ and the 3D noise sample used $\epsilon_{3D}^{\left(i\right)}$. When trying to sample $x_t$ for a newly sampled view $v$ we can then take all $X^{\left(0\right)},...,X^{\left(T-t\right)}$, and all $\epsilon_{3D}^{\left(0\right)},..., \epsilon_{3D}^{\left(T-t\right)}$ and project them to view $v$. We can then apply a sampling loop using the projections, just like we did in the original algorithm. We observe that in our setting, this method does not lead to significant improvement so we present it as an optional addition.

\section{Data Collection}
\label{sec:data_collect}
To demonstrate the merits of \ouralg{} we collected three 2D motion datasets, extracted from in-the-wild videos.

\textbf{NBA videos.} We collected about $10K$ videos from the NBA online API\footnote{\url{https://github.com/swar/nba_api}}. We then applied multi-person tracking using ByteTrack~\citep{zhang2022bytetrack}, and  AlphaPose~\citep{fang2022alphapose} for 2D human pose estimation (based on the tracking results). 
We finally processed and filtered the data by centering the people, filtering short motions, crowd motions, and motions of low quality, splitting discontinuous motions (caused typically by tracking errors), mirroring, and applying smoothing interpolations. 

\textbf{Horse jumping contests.} We collected 3 horse jumping contest videos (around 2-3 hours each) from YouTube.com. We then apply YoloV7~\citep{wang2023yolov7} for horse detection and tracking and VitPose~\citep{xu2022vitpose} trained on APT-36K~\citep{yang2022apt} for horse pose estimation. The post-processing pipeline was similar to the one described above. 

\textbf{Rhytmic-ball gymnastics}. We used the Rhythmic Gymnastics Dataset~\citep{zeng2020hybrid} to get 250 videos, about 1.5 minutes long each, of high-standard international competitions of rhythmic gymnastics performance with a ball. We followed the pipeline described for NBA videos to obtain athletes' motions and also use YoloV7 ~\citep{wang2023yolov7} for detecting bounding boxes of sports balls. We take the closest ball to the athlete at each frame and add the center of the bounding box as an additional "joint" in the motion representation.

All motions are represented as $x \in \mathbb{R}^{L\times J\times 2}$, 
where NBA is using the AlphaPose body model with $16$ joint, horses represented according to APT-36K with $17$ joints and the gymnastics dataset is represented with the COCO body model~\citep{lin2015microsoft} with $17$ joints plus additional joint for the ball. 
All 2D pose predictions are accompanied by confidence predictions per joint per frame which are used in the diffusion training process.

\section{Additional Experiments}
\label{sec:add_results}

Table~\ref{table:comparison_human36m} presents a comparison of our method with off-the-shelf SOTA methods for supervised pose lifting - MotionBERT~\citep{zhu2023motionbert}, unsupervised pose lifting - ElePose~\citep{wandt2021elepose}, and DreamFussion~\citep{poole2022dreamfusion} adaptation. \ouralg{} is on par with the lifting method for the more challenging side views.

Table~\ref{table:more_ablation} depicts an ablation study for the number of views, camera distance, and diffusion steps.

Figure~\ref{fig:user_study_screenshot} presents a screenshot from the user study presented in the paper, including the wording of the questions for each of the three aspects - \emph{Precision, Diversity, and Quality}.

\begin{table*}[t]
\vspace{10pt}
\resizebox{0.9\textwidth}{!}{
\begin{tabular}{ lcc|cc|cc|cc  }
\toprule
 & \multicolumn{2}{c}{FID$\downarrow$} & \multicolumn{2}{c}{Diversity$\rightarrow$} & \multicolumn{2}{c}{Precision$\uparrow$} & \multicolumn{2}{c}{Recall$\uparrow$} \\
 \hline
 View Angles  & All & Side  & All & Side  & All & Side  & All & Side  \\
 \hline
Human3.6M (GT)
& \multicolumn{2}{c|}{$7.34^{\pm0.18}$}
& \multicolumn{2}{c|}{$10.74^{\pm0.15}$}
& \multicolumn{2}{c|}{$0.52^{\pm0.01}$}
& \multicolumn{2}{c}{$0.91^{\pm0.005}$}
\\
\hline
ElePose
& $11.20^{\pm0.36}$ & $24.13^{\pm0.16}$
& $10.67^{\pm0.05}$ & $10.24^{\pm0.08}$
& $0.47^{\pm0.02}$ & $\mathbf{0.41^{\pm0.01}}$
& $0.80^{\pm0.01}$ & $0.25^{\pm0.01}$
\\
MotionBert
& $14.05^{\pm0.14}$ & $24.12^{\pm0.29}$
& $11.46^{\pm0.07}$ & $\mathbf{11.18^{\pm0.06}}$
& $0.32^{\pm0.01}$ & $0.21^{\pm0.01}$
& $0.88^{\pm1.21e-03}$ & $0.56^{\pm0.02}$
\\
MAS (ours)
& \multicolumn{2}{c|}{$\mathbf{15.15^{\pm0.16}}$}
& \multicolumn{2}{c|}{$11.94^{\pm0.07}$}
& \multicolumn{2}{c|}{$0.21^{\pm0.01}$}
& \multicolumn{2}{c}{$\mathbf{0.92^{\pm0.01}}$}
\\

\hline
\end{tabular}
}
\centering
\caption{\textbf{Comparison with pose lifting on Human3.6M dataset.} \ouralg{} has a competitive performance to lifting methods that were designed for this dataset. However, \ouralg{} outperforms the lifting methods when evaluated from the side view. Here, \textbf{bold} marks the best results when comparing to the side view.
}
\label{table:comparison_human36m}
\end{table*}

\begin{figure}
\includegraphics[width=\columnwidth]{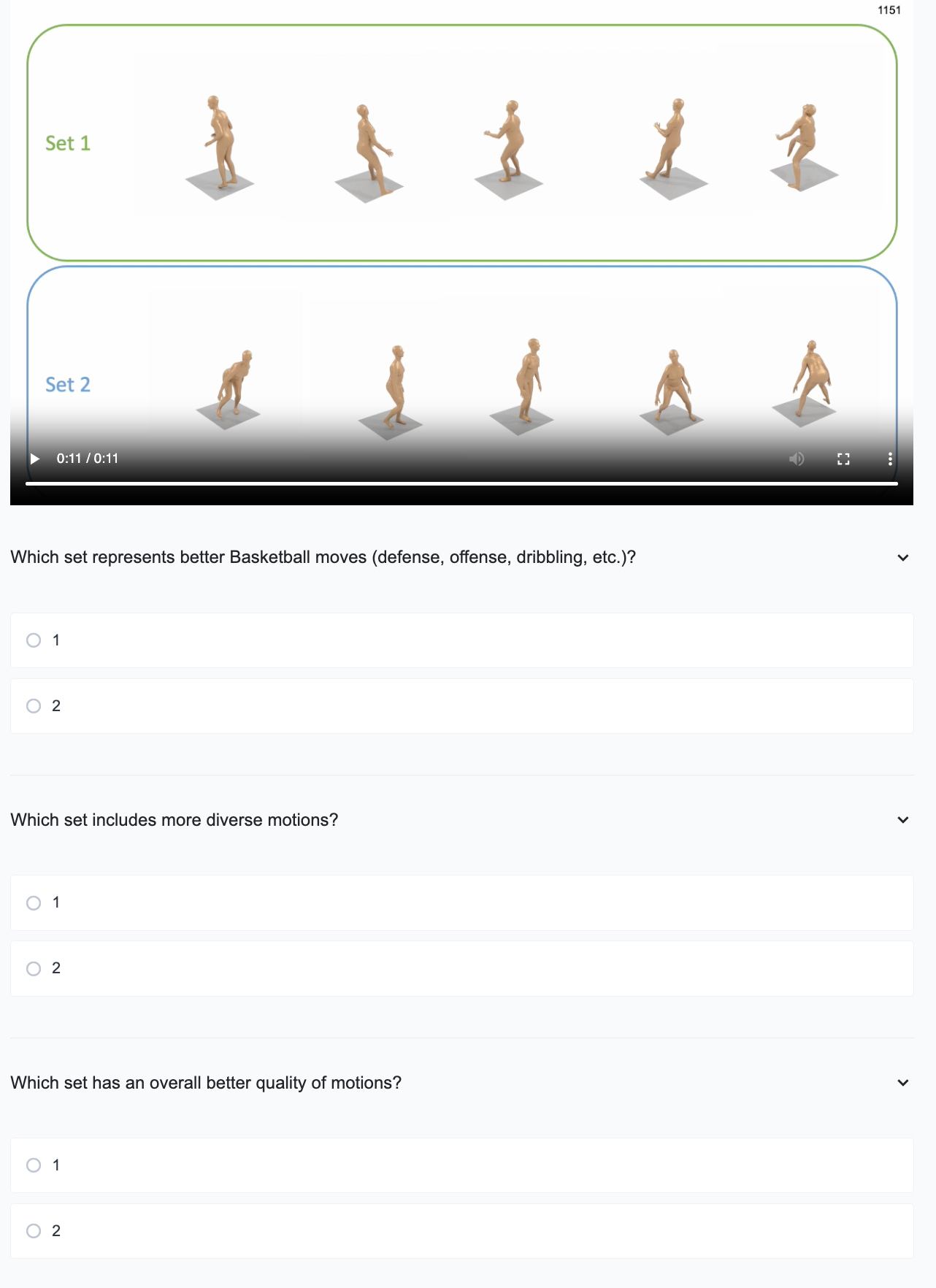}
\caption{
\textbf{NBA User study screenshot.} 
A screenshot from the user study conducted with \href{https://www.pollfish.com/}{https://www.pollfish.com/}.
} \label{fig:user_study_screenshot}
\end{figure}

\begin{table}[t]
\centering
\resizebox{1\columnwidth}{!}{
\begin{tabular}{lcccc}
\toprule
& FID$\downarrow$ & Diversity$\rightarrow$ & Precision$\uparrow$ & Recall$\uparrow$ \\
\hline
Ground Truth 
& $1.05^{\pm.02}$  
& $8.97^{\pm.05} $
& $0.73^{\pm.01}$ 
& $0.73^{\pm.01} $
\\
\hline
\#views=2 ($120^\circ$) 
&$5.17^{\pm.12}$
&$9.86^{\pm.04}$
&$0.42^{\pm.03}$
&$0.77^{\pm.01}$
\\
\#views=3
&$4.01^{\pm.15}$
&$9.55^{\pm.04}$
&$0.53^{\pm.02}$
&$0.70^{\pm.01}$
\\
\#views=5 (ours)
& $\mathbf{3.92^{\pm.15}}$
& $\mathbf{9.47^{\pm.07}}$
& $\mathbf{0.56^{\pm.03}}$
& $\mathbf{0.67^{\pm.01}}$
\\
\#views=9
&$\underline{3.94^{\pm.12}}$
&$\underline{9.48^{\pm.05}}$
&$\mathbf{0.56^{\pm.02}}$
&$\mathbf{0.67^{\pm.01}}$
\\
\#views=21
&$\underline{3.94^{\pm.12}}$
&$\underline{9.48^{\pm.05}}$
&$\mathbf{0.56^{\pm.02}}$
&$\mathbf{0.67^{\pm.01}}$
\\

\hline
camera dist=2[m]
&$7.59^{\pm.13}$
&$\mathbf{9.36^{\pm.05}}$
&$0.46^{\pm.01}$
&$0.46^{\pm.01}$
\\
camera dist=3[m]
&$4.78^{\pm.11}$
&$\underline{9.45^{\pm.05}}$
&$0.53^{\pm.01}$
&$0.63^{\pm.02}$
\\
camera dist=5[m]
&$\underline{3.99^{\pm.12}}$
&$9.47^{\pm.05}$
&$\mathbf{0.57^{\pm.02}}$
&$\mathbf{0.67^{\pm.01}}$
\\
camera dist=7[m] (ours)
& $\mathbf{3.92^{\pm.15}}$
& $9.47^{\pm.07}$
& $\underline{0.56^{\pm.03}}$
& $\mathbf{0.67^{\pm.01}}$
\\
camera dist=11[m]
&$4.04^{\pm.13}$
&$9.48^{\pm.05}$
&$0.55^{\pm.02}$
&$0.66^{\pm.01}$
\\
camera dist=30[m]
&$4.29^{\pm.14}$
&$9.49^{\pm.05}$
&$0.55^{\pm.02}$
&$0.65^{\pm.01}$
\\

\hline
diff steps=20
&$5.14^{\pm.13}$
&$9.04^{\pm.01}$
&$\mathbf{0.68^{\pm.01}}$
&$0.42^{\pm.01}$
\\
diff steps=50
&$5.49^{\pm.13}$
&$\mathbf{8.99^{\pm.04}}$
&$\mathbf{0.68^{\pm.02}}$
&$0.36^{\pm.01}$
\\
diff steps=100 (ours)
& $\mathbf{3.92^{\pm.15}}$
& $9.47^{\pm.07}$
& $0.56^{\pm.03}$
& $\mathbf{0.67^{\pm.01}}$
\\

\bottomrule
\end{tabular}
} 
\vspace{-10pt}
\caption{
\small
\textbf{NBA Dataset Ablations.} 
Performance saturates for number of views $\geq 5$; Optimal performance achieved at camera distance (dist) around $7$ meters; Fewer diffusion steps harm recall and FID.
}
\label{table:more_ablation}
\end{table}

\newtheorem{theorem}{Theorem}
\section{Gradient Update Formula}
\label{appendix:grad_update}

In order to clarify the difference between SDS and our method, we calculate the gradient update formula w.r.t our optimized loss. Denote by $X^{\left(i\right)}$ the optimizing motion at iteration $i$.
When differentiating our loss w.r.t $X^{\left(i\right)}$  we get:
\begin{align}
\nabla_{X^{\left(i\right)}}\lVert P\left(X^{\left(i\right)}\right)-\hat{x_{0}}\rVert	_{2}^{2} & \\=\left(P\left(X^{\left(i\right)}\right)-\frac{x_{t}-\sqrt{1-\bar{\alpha}_{t}}\epsilon_{\phi}\left(x_{t}\right)}{\sqrt{\bar{\alpha}_{t}}}\right)\frac{\partial p}{\partial X^{\left(i\right)}}
\end{align}
which is clearly differers from $\nabla\mathcal{L}_\text{SDS}$. Let us observe substituting our $x_t$ sampling with a simple forward diffusion: $x_t=\sqrt{\bar{\alpha}_t} P\left( X^{\left(i-1\right)}\right)+\left(\sqrt{1-\bar{\alpha}}\right)\varepsilon$ - as used in DreamFusion. (This formulation is also analyzed in HIFA~\citep{zhu2023hifa}):
 
{\tiny 
\begin{align}
\nabla_{X^{\left(i\right)}}\lVert P\left(X^{\left(i\right)}\right)-\hat{x_{0}}\rVert  = \\ \left(P\left(X^{\left(i\right)}\right)-\frac{x_{t}-\sqrt{1-\bar{\alpha}_{t}}\epsilon_{\phi}\left(x_{t}\right)}{\sqrt{\bar{\alpha}_{t}}}\right)\frac{\partial p}{\partial X_{i}}= \\
 \left(P\left(X^{\left(i\right)}\right)-\frac{\sqrt{\bar{\alpha}_{t}}P\left(X^{\left(i-1\right)}\right)+\sqrt{1-\bar{\alpha}_{t}}\varepsilon-\sqrt{1-\bar{\alpha_{t}}}\epsilon_{\phi}\left(x_{t}\right)}{\sqrt{\bar{\alpha_{t}}}}\right)\frac{\partial p}{\partial X_{i}}= \\
 \left(P\left(X^{\left(i\right)}\right)-P\left(X^{\left(i-1\right)}\right)+\frac{\sqrt{1-\bar{\alpha}_{t}}}{\sqrt{\bar{\alpha_{t}}}}\left(\varepsilon-\epsilon_{\phi}\left(x_{t}\right)\right)\right)\frac{\partial p}{\partial X_{i}}
\end{align}
} %

 If we observe the first iteration of optimization, we have $X^{\left(i\right)}=  X^{\left(i-1\right)}$ so we get:
\begin{equation}
    \nabla_{X}\lVert P\left(X\right)-\hat{x_{0}}\rVert_2^2 
=\frac{\sqrt{1-\bar{\alpha}_{t}}}{\sqrt{\bar{\alpha_{t}}}}\left(\varepsilon-\epsilon_{\phi}\left(x_{t}\right)\right)\frac{\partial p}{\partial X}
\end{equation}
This shows that SDS loss is a special case of our loss when sampling $x_t=  \sqrt{\bar{\alpha}_t} P\left( X^{\left(i-1\right)}\right)+\left(\sqrt{1-\bar{\alpha}}\right)\varepsilon$ (where $\varepsilon \sim \mathcal{N}\left(0,I\right)$), and applying only a single optimization step (after the first step, $X^{\left(i\right)}\ne X^{\left(i-1\right)}$).
\section{Theorems}
\label{appendix:theorems}

\begin{theorem}
\label{theorem: multiview noise distribution}
Let $\varepsilon=\left(\begin{matrix}x_{\varepsilon}\\
y_{\varepsilon}\\
z_{\varepsilon}
\end{matrix}\right)\sim \mathcal{N}\left(0,I_{3\times3}\right)$ and let $P\in \mathbb{R}^{2\times 3}$ be an orthogonal projection matrix, then $P\cdot\varepsilon\sim\mathcal{N}\left(0,I_{2\times2}\right)$.
\end{theorem}
\begin{proof}
First, $P\cdot\varepsilon$ has a normal distribution as a linear
combination of normal variables.\\
In addition, $\mathbb{E}\left[P\cdot\varepsilon\right]=P\cdot\mathbb{E}\left[\varepsilon\right]=0$.\\
Now we will prove that $\mathrm{Var}\left[P\cdot\varepsilon\right]=I_{2\times2}$:\\
Denote $O=\left(\begin{matrix}1 & 0 & 0\\
0 & 1 & 0
\end{matrix}\right)$ then we know that $P=O\cdot P^{\prime}$ where $P^{\prime}$ is a
rotation matrix, i.e. $P^{\prime}\cdot\left(P^{\prime}\right)^{T}=I_{2\times2}$
. Then 
\begin{align*}
P\cdot P^{T} & =\left(O\cdot P^{\prime}\right)\left(O\cdot P^{\prime}\right)^{T}=\\ O\cdot\overset{I}{\overbrace{\left(P^{\prime}P^{\prime}{}^{T}\right)}}O^{T}=OO^{T}=I
\end{align*}
 Furthermore, $\mathbb{E}\left[\varepsilon\cdot\varepsilon^{T}\right]=\mathbb{E}\left[\varepsilon\cdot\varepsilon^{T}\right]-\overset{0}{\overbrace{\mathbb{E}\left[\varepsilon\right]\mathbb{E}\left[\varepsilon\right]^{T}}}=\mathrm{Var}\left[\varepsilon\right]=I$.\\
Therefore:
\begin{align*}
\mathrm{Var}\left[P\cdot\varepsilon\right]  =\mathbb{E}\left[\left(P\cdot\varepsilon\right)\left(P\cdot\varepsilon\right)^{T}\right]-\overset{0}{\overbrace{\mathbb{E}\left[P\cdot\varepsilon\right]}}\cdot\overset{0}{\overbrace{\mathbb{E}\left[P\cdot\varepsilon\right]^{T}}}=\\
  \mathbb{E}\left[P\cdot\varepsilon\cdot\varepsilon^{T}\cdot P^{T}\right]=\\ P\cdot\overset{I}{\overbrace{\mathbb{E}\left[\varepsilon\cdot\varepsilon^{T}\right]}}\cdot P^{T}=P\cdot P^{T}=I
\end{align*}
\end{proof}

\begin{theorem}
\label{theorem: orthographic vs perspective}
Let $X\in \mathbb{R}^3$, and denote by $p_\text{orth}\left(X\right),p_\text{pers}\left(X\right)$  the orthographic and perspective projections of $X$ to the same view, respectively. We assume that the subject is centered in the origin and is bounded in a sphere with radius $1$ ($\left\Vert X\right\Vert _{\infty}\leq1$). We also assume the perspective projection is dome from distance $d$ from the origin. Then $\left\Vert p_\text{orth}\left(X\right)-p_\text{pers}\left(X\right) \right\Vert _{\infty} = O\left(\frac{1}{d-1}\right)$.
\end{theorem}
\begin{proof}
First, denote the rotation matrix that corresponds to the view by
$R\in\mathbb{R}^{3\times3}$ and $R_{xy}=\left(\begin{matrix}1 & 0 & 0\\
0 & 1 & 0
\end{matrix}\right)\cdot R$, $R_{z}=\left(\begin{matrix}0 & 0 & 1\end{matrix}\right)\cdot R$.
Then 
\[
p_{\text{orth}}\left(X\right)=R_{xy}\cdot X,p_{\text{pers}}\left(X\right)=\frac{R_{xy}\cdot X}{d+R_{z}\cdot X}\cdot d
\]
So
\begin{align*}
p_{\text{orth}}\left(X\right)-p_{\text{pers}}\left(X\right)   =\\ R_{xy}\cdot X-\frac{R_{xy}\cdot X}{d+R_{z}\cdot X}\cdot d=\frac{R_{xy}\cdot X\cdot\left(d+R_{z}\cdot X-d\right)}{d+R_{z}\cdot X}=\\
  \frac{R_{xy}\cdot X\cdot R_{z}\cdot X}{d+R_{z}\cdot X}
\end{align*}
Assume $\left\Vert X\right\Vert _{\infty}\leq 1$, then $\left\Vert\frac{R_{xy}\cdot X\cdot R_{z}\cdot X}{d+R_{z}\cdot X}\right\Vert_\infty \leq\frac{1}{d-1}$.
\end{proof}

\end{document}